\documentclass{article} 
\usepackage{hyperref}
\usepackage[accepted]{icml2020}

\usepackage{multirow}
\usepackage{amsmath}
\usepackage{amsfonts}
\usepackage{amssymb}
\usepackage{amsthm}

\usepackage{times}
\usepackage{url}
\usepackage{mathrsfs}
\usepackage{booktabs}

\usepackage{mathdef}

\usepackage{lipsum} 
\usepackage{graphicx}

\newtheorem{theorem}{Theorem}

\newtheorem{lemma}{Lemma}

\global\long\def\L{L(\x,\vv\theta)}
\global\long\def\Lprime{L(\x_i^\prime,\vv\theta)}

\global\long\def\Laug{\hat L_{\P, m}(\x, \vv\theta)}

\icmltitlerunning{MaxUp: A Simple Way to Improve Generalization of 
	Neural Network Training} 

\begin{document}
\twocolumn[

\icmltitle{MaxUp: 
A Simple Way to Improve Generalization of 
Neural Network Training}
\icmlsetsymbol{equal}{*}
\begin{icmlauthorlist}
\icmlauthor{Chengyue Gong$^*$}{ut}
\icmlauthor{Tongzheng Ren$^*$}{ut}
\icmlauthor{Mao Ye}{ut}
\icmlauthor{Qiang Liu}{ut}
\end{icmlauthorlist}
\icmlaffiliation{ut}{UT Austin}
\icmlcorrespondingauthor{Chengyue Gong}{cygong@cs.utexas.edu}
\vskip 0.3in
]

\printAffiliationsAndNotice{\icmlEqualContribution}
\date{\today}

\begin{abstract}
We propose \emph{MaxUp}, an embarrassingly simple, highly effective technique for improving the generalization performance of machine learning models, especially  
deep neural networks. 
The idea is to generate a set of augmented data with some random perturbations or transforms, and minimize the maximum, or worst case loss over the augmented data. 
By doing so, we implicitly introduce a smoothness or robustness regularization against the random perturbations, and hence improve the generation performance. 
For example, in the case of Gaussian perturbation, 
 \emph{MaxUp} is asymptotically equivalent to using 
 the  gradient norm of the loss as a penalty to encourage smoothness.
We  test \emph{MaxUp} on a range of tasks, 
including image classification, language modeling, and adversarial  certification, on which \emph{MaxUp} consistently outperforms the existing best baseline methods, without introducing substantial computational overhead. 
In particular,  we improve ImageNet classification from the  state-of-the-art top-1 accuracy $85.5\%$ without extra data to $85.8\%$. 
Code will be released soon.
\end{abstract}

\section{Introduction}

A central theme of machine learning
is to alleviate the issue of overfitting, 
improving the generalization performance on testing data. 
This is often achieved by leveraging important prior knowledge of the models and data of interest. 
For example, the
regularization-based methods introduce penalty on the complexity of the model, 
which often amount to enforcing certain smoothness properties. 
Data augmentation techniques, on the other hand, 
leverage important invariance properties of the data 
(such as the shift and rotation invariance of images) 
to improve performance. 
Novel approaches that exploit important knowledge of the models and data hold the potential of substantially improving the performance of machine learning systems.

We propose \emph{MaxUp}, a simple yet powerful training method
to improve the generalization performance
and alleviate the over-fitting issue. Different from standard methods that minimize the average risk on the observed data, \emph{MaxUp} generates a set of random perturbations or transforms of each observed data point, and minimizes the average risk of  the \emph{worst} augmented data of each data point. 
This allows us to enforce robustness against the random perturbations and transforms, 
and hence improve the generalization performance. \emph{MaxUp} can easily leverage arbitrary  state-of-the-art data augmentation schemes \citep[e.g.][]{zhang2018mixup, devries2017cutout, cubuk2018autoaugment}, and substantially improves  over them 
by minimizing the worst (instead of average) risks on the augmented data, without adding significant computational ahead. 

Theoretically,
in the case of Gaussian perturbation, 
we show that 
\emph{MaxUp} effectively introduces 
a \emph{gradient-norm regularization term} that serves to encourage smoothness of the loss function, 
which does not appear in  standard data augmentation methods that minimize the average risk. 

\emph{MaxUp} can be viewed as a ``lightweight'' variant of adversarial training against adversarial input pertubrations \citep[e.g.][]{tramer2017ensemble, madry2017towards}, 
but is mainly designed to improve the generalization on the clean data, instead of robustness on perturbed data (although \emph{MaxUp} does also increase the adversarial robustness in Gaussian adversarial certification as we shown in our experiments (Section~\ref{sec:adversarial})).   
In addition, 
compared with standard adversarial training methods such as projected gradient descent (PGD) \citep{madry2017towards},  
\emph{MaxUp} is 
much simpler and computationally much faster, and can be easily adapted to increase various  robustness defined by the corresponding data augmentation schemes. 

We test \emph{MaxUp} on three challenging tasks: image classification, language modeling, and certified defense  against adversarial examples~\citep{cohen2019certified}. 
We find that \emph{MaxUp} can leverage the 
different state-of-the-art 
 data augmentation methods and boost their performance to achieve new state-of-the-art on a range of tasks, datasets, and neural architectures. 
In particular, 
we set up a new state-of-the-art result on ImageNet classification without extra data,  
which improves the best $85.5\%$ top1 accuracy 
 by \citet{xie2019adversarial} to $85.8\%$. 
For the adversarial certification task, we find \emph{Maxup} allows us to train more verifiably robust classifiers than 
prior arts such as the PGD-based adversarial training proposed by \citet{salman2019provably}.


\begin{algorithm*}
\caption{\emph{MaxUp} with Stochastic Gradient Descent}
\begin{algorithmic}
\STATE {\bfseries Input:} Dataset $\mathcal{D}_n = \{\x_i\}_{i=1}^n$; transformation distribution $\mathbb{P}(\cdot|\x)$; number of augmented data $m$; initialization $\th_0$;  
SGD parameters (batch size, step size $\eta$, etc).
\REPEAT 
    \STATE Draw a mini-batch $\mathcal M$ from $\Dn$, 
    and update $\th$ via 
    \begin{align*}
        \th \gets \th - \eta 
         \E_{\x \sim \mathcal M} 
        \left [
        \nabla_{\th} \left ( \max_{i\in [m]} L(\x_i^\prime, \vv\theta)\right ) \right], 
    \end{align*}
    where $\{\x_{i}'\}_{i=1}^m$ are drawn \emph{i.i.d.} from $\mathbb{P}(\cdot|\x)$ for each $\x$ in the mini batch $\mathcal M$. See Equation~\ref{equ:grad}. 
\UNTIL{convergence}
\label{alg:main}
\end{algorithmic}
\end{algorithm*}

\section{Main Method} 
\label{sec:method}
We start with introducing the main idea of \emph{MaxUp}, and then discuss its effect of introducing smoothness regularization in Section~\ref{sec:smooth}.  

\paragraph{ERM}
Giving a dataset $\Dn =\{\x_i\}_{i=1}^n$, 
learning often reduces to a form of 
empirical risk minimization (ERM):
\begin{align}
\min_{\th} \E_{\x\sim\Dn} \left [L(\x, \vv\theta) \right],
\end{align}
where $\th$ is a parameter of interest (e.g., the weights of a neural network), and $L(\x, \vv\theta)$ denotes the loss associated with data point $\x$. 
A key issue of ERM is the risk of overfitting, 
especially when the data information is insufficient. 


\paragraph{\emph{MaxUp}} 
We propose \emph{MaxUp} to alleviate overfitting. 
The idea is to generate a set of random augmented data and minimize the maximum loss over the augmented data. 

Formally, for each data point $\x$ in $\Dn$, 
we generate a set of perturbed data points $\{\x_i'\}_{i=1}^m$ that are similar to $\x$, and estimate $\th$  
by minimizing the maximum loss over $\{\x_i'\}$: 
\begin{align}\label{equ:maxup1}
\text{\emph{MaxUp:}} &&   \min_{\th} 
\mathbb{E}_{\x\sim \Dn} 
    \left [ \max_{i\in [m]} L(\x_i^\prime, \th)
    \right].
\end{align}
This loss can be easily minimized with stochastic gradient descent (SGD). Note that the gradient of the maximum loss is simply the gradient of the worst copy, that is, 
\begin{align} \label{equ:grad}
\nabla_{\vv\th} \left(\max_{i\in [m]} L(\x_i^\prime, \th)\right)
= \nabla_{\vv\th} L(\x_{i^*}^\prime, \th),
\end{align}
where $i^* =\argmax_{i\in [m]} L(\x_i^\prime, \th)$. 
This yields a simple and practical algorithm shown in Algorithm~\ref{alg:main}. 

In our work, we assume the augmented data $\{\x_i'\}_{i=1}^m$ is \emph{i.i.d.} generated from a  distribution $\P(\cdot|\x)$. 
The $\P(\cdot|\x)$ can be based on small perturbations around $\x$, e.g., $\P(\cdot|\x) = \normal(\x, \sigma^2 \vv I)$, the Gaussian distribution with mean $\x$ and isotropic variance $\sigma^2$. 
The $\P(\cdot |\x)$ can also be 
constructed based on invariant data transformations that are widely used in the data  augmentation literature,  
such as random crops, equalizing, rotations, and clips for images
\citep[see e.g][]{cubuk2018autoaugment, devries2017cutout, cubuk2019randaugment}. 

\subsection{\emph{MaxUp} as a Smoothness Regularization}\label{sec:smooth} 
We provide a theoretical interpretation of \emph{Maxup} 
as introducing a \emph{gradient-norm regularization} to the original ERM objective to encourage smoothness. 
Here we consider the simple case of isotropic Gaussian perturbation, when 
$\mathbb{P}(\cdot|\x) = \normal(\x, \sigma^2 \vv I)$. 
For simplifying notation, we define 
\begin{align} \label{equ:maxup_risk}
\tilde L_{\P, m}(\x, \vv\theta) := \E_{\{\x_i'\}_{i=1}^m\sim \P(\cdot|\x)^m} \left [ \max_{i\in [m]}
L(\x_i^\prime, \vv\theta)
\right],
\end{align}
which represents the expected \emph{MaxUp} risk of data point $\x$ with $m$ augmented copies. 



\begin{theorem}[MaxUp as Gradient-Norm Regularization]
Consider $\tilde{L}_{\P, m}(\x, \vv\theta)$ defined in \eqref{equ:maxup_risk} with $\mathbb{P}(\cdot|\x) = \normal(\x, \sigma^2\vv  I)$. 
Assume $L(\x, \vv\theta)$ is second-order differentiable w.r.t. $\x$. 
Then \begin{align}
\tilde{L}_{\P, m} &(\x, \vv\theta)  = \nonumber L(\x, \vv\theta) + c_{m, \sigma} \left \| \nabla_{\x} L(\x, \vv\theta) \right\|_2 + \mathbf O(\sigma^2), 
\end{align}
\label{thm:smooth_regularization}
where $c_{m,\sigma}$ is a constant and 
$c_{m, \sigma} =\mathbf \Theta( \sigma \sqrt{\log m})$, where $\mathbf \Theta(\cdot)$ denotes the big-Theta notation. 

\end{theorem}

Theorem \ref{thm:smooth_regularization} shows that, the expected \emph{MaxUp} risk can be viewed as introducing a
Lipschitz-like regularization with the gradient norm $\|\nabla_x \L\|_2$, which encourages the smoothness of $\L$ w.r.t. the input $\x$. 
The strength of the regularization is controlled by $c_{m,\sigma}$, which depends on the number of samples $m$ and perturbation magnitude  $\sigma$. 
\begin{proof}
Using Taylor expansion, we have 
\begin{align*}
    & \tilde L_{\mathbb{P},m}(\x, \th)  \\
    &=  
    \E
    \left [\max_{i\in [m]} L (\x_i^\prime, \vv\theta) \right ]\\
    & = L (\x, \vv\theta) + \mathbb{E}
    \left [ \max_{i\in [m]} \left(L (\x_i^\prime, \vv\theta) - L (\x, \vv\theta)\right) \right ]\\
     &= \L +  \mathbb{E}
     \left[ \max_{i\in[q]}\langle \nabla_{\x} \L, \vv z_i\rangle \right]+ \mathbf O(\sigma^2),
\end{align*} 
where we assume $\vv z_i = \x_i' - \x$, which follows $\normal(0, \sigma^2 \vv I)$. 
The rest of the proof is 
due to the Lemma~\ref{lem:maxGauss} below. 
\end{proof}


\begin{lemma}\label{lem:maxGauss}
Let $\vv g$  be a fixed vector in $\R^d$,
and $\{\vv z_i\}_{i=1}^m$ are $m$ i.i.d. random variables from $\normal(0, \sigma^2 \vv I)$. 
We have 
\begin{align*}
    \mathbb{E} \left [ \max_{i\in[m]}\langle \vv g, \vv z_i\rangle \right ] 
    = c_{m, \sigma} \|\vv g\|_2, 
\end{align*}
where $c_{m, \sigma} 
= \mathbf \Theta\left (\sigma \sqrt{\log m}\right).$ 
\end{lemma}
\begin{proof}
Define $y_i = \langle \vv g, \vv z_i\rangle/\norm{\vv g}_2$. 
Then $\{y_i\}_{i=1}^m$ is \emph{i.i.d.} from $\normal(0, \sigma^2)$. 
Therefore, $c_{m,\sigma}=\E[\max_{i\in [m]} y_i]$, 
which is well known to be $\mathbf \Theta(\sigma\sqrt{\log m})$.
See e.g., \citet{orabona2015optimal, kamath2015bounds} for bounds related to $\E[\max_{i\in [m]} y_i]$. 
More specifically, 
we have 
$0.23\sigma \sqrt{\log m} \leq  c_{m,\sigma} \leq \sqrt{2}\sigma \sqrt{\log m}$ following \citet{kamath2015bounds}. 
\end{proof}



\section{Related Methods and Discussion}
\emph{MaxUp} is closely related 
to both data augmentation and adversarial training.  
It can be viewed as an \emph{adversarial variant of data augmentation}, in that it minimizes the worse case loss on the perturbed data, instead of an average loss like typical data augmentation methods. 
\emph{MaxUp} can also be viewed as a \emph{``lightweight'' variant of adversarial training}, in that
the maximum loss is calculated by simple random sampling, instead of more accurate gradient-based optimizers for finding the adversarial loss, 
such as projected gradient descent (PGD);  
\emph{MaxUp} is much simpler and faster than the PGD-based adversarial training, and is more suitable for our purpose of alleviating over-fitting on clean data (instead of adversarial defense). 
We now elaborate on these connections in depth. 
\subsection{Data Augmentation}
Data augmentation has been widely used in machine learning, especially on image data which admits a rich set of invariance transforms (e.g. translation, rotation, random cropping). 
Recent augmentation techniques, such as  MixUp~\citep{zhang2018mixup}, 
CutMix~\citep{yun2019cutmix} and manifold MixUp~\citep{verma2018manifold} 
have been found highly useful in training deep neural networks, 
especially in achieving state-of-the-art results on important image classification benchmarks such as SVHN, CIFAR and ImageNet.  
More recently,
more advanced methods have been developed to 
find the optimal data augmentation policies 
using reinforcement learning or adversarial generative network~\citep[e.g.][]{cubuk2018autoaugment, cubuk2019randaugment, zhang2019adversarial}.

\emph{MaxUp} can easily leverage these advanced
 data augmentation techniques to achieve good performance. 
 The key difference, however, is that \emph{MaxUp} in \eqref{equ:maxup1} minimizes the \emph{maximum loss} on the augmented data, while typical data augmentation methods minimize the \emph{average loss}, that is, 
\begin{align}\label{equ:SumUp}
    \min_{\th}\mathbb{E}_{ \x \sim \Dn} 
    \left [ \frac{1}{m}\sum_{i=1}^m \Lprime \right],
\end{align}
which we refer to as \emph{standard data augmentation} throughout the paper. 
It turns out \eqref{equ:maxup1} and \eqref{equ:SumUp} 
behave very different as regularization mechanisms, in that \eqref{equ:SumUp} does not introduce the gradient-norm regularization as \eqref{equ:maxup1},  
and hence does not have the benefit of having gradient-norm regularization. 
This is because the first-order term in the Taylor expansion is canceled out due to the averaging in \eqref{equ:SumUp}.  

Specifically, 
let $\mathbb{P}(\cdot|\x)$ be any distribution whose expectation is $\x$ and $\L$ is second-order differentiable w.r.t $\x$. 
Define the expected loss related to \eqref{equ:SumUp} on data point $\x$: 
\begin{align}
    \Laug := \E_{\{\x_i'\}_{i=1}^m\sim \P(\cdot|\x)^m} \left [ \frac{1}{m} \sum_{i=1}^m \Lprime \right]. 
\end{align}
Then with a simple Taylor expansion, we have 
\begin{align}
    \Laug = \nonumber \L + \mathbf O(\sigma^2),  
\end{align}
which misses the gradient-norm regularization term when compared with \emph{MaxUp} decomposition in Theorem~\ref{thm:smooth_regularization}. 

Note that the \emph{MaxUp} update is computationally \emph{ faster} than the 
solving \eqref{equ:SumUp} with the same $m$,
because we only need to backpropagate on the worst augmented copy for each data point (see Equation~\ref{equ:grad}),  
while solving \eqref{equ:SumUp} requires to backpropagate 
on all the $m$ copies at each iteration.

\subsection{Adversarial Training} 
Adversarial training has been developed to 
defense various adversarial attacks on the data inputs \citep{madry2017towards}.
It estimates $\th$ by solving the following problem:
\begin{align}\label{equ:adv}
    \min_\th \mathbb{E}_{\vv{x}\sim \mathcal{D}_n}  
    \left [ \max_{\vv x^\prime \in \mathcal{B}(\vv x, r)} L(\x^\prime, \th) \right] ,
\end{align}
where $\mathcal{B}(\x, r)$ represents a ball centered at $\x$ with radius $r$ under some metrics (e.g. $\ell_0$, $\ell_1$, $\ell_2$, or $\ell_{\infty}$ distances). The inner maximization is often solved by running projected gradient descent (PGD) for a number of iterations. 

\emph{MaxUp} in \eqref{equ:maxup1}  can be roughly viewed as 
solving the inner adversarial maximization problem in \eqref{equ:adv} 
using a ``{mild}'', or {``lightweight''}  optimizer by randomly drawing $m$ points from $\P(\cdot | \x)$ and finding the best. 
Such mild adversarial optimization increases the robustness against the random perturbation it introduces, 
and hence enhance the generalization performance. 
Adversarial ideas have also been used to improvement generalization  in a series of recent works \citep[e.g.,][]{xie2019adversarial, zhu2019freelb}.  

Different from our method, 
typical adversarial training methods, especially these based PGD~\citep{madry2017towards},   
tend to 
solve the adversarial  optimization much more \emph{aggressively} to achieve higher robustness, but at the cost of scarifying the accuracy on clean data.  
There has been shown a clear trade-off between the 
accuracy of a classifier on clean data and 
its robustness against adversarial attacks \citep[see e.g.,][]{tsipras2018robustness, pmlr-v97-zhang19p, yin2019rademacher, schmidt2018adversarially}. 
By using a mild adversarial optimizer, \emph{MaxUp} strikes a better balance between the
accuracy on clean data and adversarial robustness.  


Besides,  
\emph{MaxUp} is much more computationally efficient than PGD-based adversarial training, because it 
does not introduce additional back-propagation steps as PGD. 
In practice, \emph{MaxUp} can be equipped with various complex data augmentation methods (in which case $\mathbb{P}(\cdot|\x)$ can be discrete distributions), while PGD-based adversarial training mostly focuses on perturbations in $\ell_p$ balls. 


\subsection{Online Hard Example Mining}
Online hard example mining  (OHEM) \citep{shrivastava2016hardmining}
is a training method originally developed for region-based objective detection, 
which improves the performance of neural networks by 
picking the hardest examples within mini batches of stochastic gradient descent (SGD). It can be viewed as
running SGD for minimizing the following expected loss 
$$
\min_{\th}  
\E_{\mathcal M} 
\left [ 
\max_{\x\in \mathcal M} L(\x, \th)
\right ], 
$$
which amounts to randomly picking a mini-batch $\mathcal M$ at each iteration and minimizing the loss of the hardest example within 
$\mathcal M$. 
By doing so, OHEM  
can focus more on the hard examples and hence improves the  performance on borderline cases. This makes OHEM particularly useful for  
class-imbalance tasks, e.g. object detection~\citep{shrivastava2016hardmining}, person re-identification~\citep{luo2019bag}. 
 

Different with \emph{MaxUp},
the hardest examples in OHEM are selected in mini-batches consisting of independently selected examples, with no special correlation or similarity.  
Mathematically, it can be viewed as reweighing the data distribution to emphasize harder instances. This is substantially different from \emph{MaxUp}, which is designed to enforce the robustness against existing random data augmentation/perturbation schemes. 

 \section{Experiments}

We  test our method using both image classification and language modeling
for which 
a variety of strong regularization techniques and data augmentation methods have been proposed.
We show that 
\emph{MaxUp} can outperform all of these methods on the most 
challenging datasets (e.g. ImageNet, Penn Treebank, and Wikitext-2) and  state-of-the-art models (e.g. ResNet, EfficientNet, AWD-LSTM). 
In addition, we apply our method to adversarial certification 
 via Gaussian smoothing \citep{cohen2019certified},  for which we find that \emph{MaxUp} can outperform both the augmented data baseline and PGD-based adversarial training baseline. 

For all the tasks, if training from scratch, we first train the model with standard data augmentation with 5 epochs and then switch to \emph{MaxUp}.

\paragraph{Time and Memory Cost}
\emph{MaxUp} only slightly increase the time and memory cost compared with standard training.  
During \emph{MaxUp}, 
we only need to find the worst instance out of the $m$ augmented copies through forward-propagation, and 
then 
only back-propagate on the worst instance. 
Therefore, the additional cost of \emph{MaxUp} over standard training is $m$ forward-propagation, 
which 
introduces no significant overhead on both memory and time cost.

\subsection{ImageNet}
We evaluate \emph{MaxUp} on  ILSVRC2012, a subset of ImageNet classification dataset~\citep{deng2009imagenet}.
This dataset contains around 1.3 million training images and 50,000 validation images. 
We follow the standard data processing pipeline including scale and aspect ratio distortions, random crops, and horizontal flips in training. During the evaluation, we  only use the single-crop setting.

\begin{table}[t]
    \centering
    \begin{tabular}{l|cc}
        \hline
        Method & Top-1 error & Top-5 error \\
        \hline
        Vanilla\tiny{~\citep{he2016resnet}} & 76.3 & - \\
        \hline
        Dropout\tiny{~\citep{srivastava2014dropout}} & 76.8 & 93.4 \\
        DropPath\tiny{~\citep{larsson2016fractalnet}} & 77.1 & 93.5 \\
        Manifold Mixup\tiny{~\citep{verma2018manifold}} & 77.5 & 93.8 \\
        AutoAugment\tiny{~\citep{cubuk2018autoaugment}} & 77.6 & 93.8 \\
        Mixup\tiny{~\citep{zhang2018mixup}} & 77.9 & 93.9 \\
        DropBlock\tiny{~\citep{ghiasi2018dropblock}} & 78.3 & 94.1 \\
        CutMix\tiny{~\citep{yun2019cutmix}} & 78.6 & 94.0 \\
        \hline
        \emph{MaxUp}+CutMix & \bf{78.9} & \bf{94.2} \\
        \hline
    \end{tabular}
    \caption{Summary of top1 and top5 accuracies on the validation set of ImageNet for ResNet-50.}
    \label{tab:resnet50}
\end{table}

\begin{table*}[h]
    \centering
    \begin{tabular}{l|c|c|cc}
        \hline
        Model & Model Size & FLOPs & +CutMix  (\%) & +\emph{MaxUp}+CutMix (\%)\\
        \hline
        ResNet-101 & 44.55M & 7.85G & 79.83 & \bf{80.26} \\
        \hline
        ProxylessNet-CPU & 7.12M & 481M & 75.32 & \bf{75.65} \\
        ProxylessNet-GPU & 4.36M & 470M & 75.08 & \bf{75.42} \\
        ProxylessNet-Mobile $\times 1.4$ & 6.86M & 603M & 76.71 & \bf{77.17} \\
        \hline
        EfficientNet-B7 & 66.35M & 38.20G & ~85.22$^*$ & ~\bf{85.45}$^*$ \\
        Fix-EfficientNet-B8 & 87.42M & 101.79G & ~85.57$^*$ & ~\bf{85.80}$^*$ \\
        \hline
    \end{tabular}
    \caption{Top1 accuracies of different models on the validation set of ImageNet 2012.
    The ``$*$'' indicates that \emph{MaxUp} is applied to the pre-trained model and trained for $5$ epochs.}
    \label{tab:imagenet}
\end{table*}

\paragraph{Implementation Details} 

We test \emph{MaxUp} 
with $\P(\cdot|\x)$  defined by 
the CutMix  data augmentation technique~\citep{yun2019cutmix} (referred to as \emph{MaxUp}+CutMix). 
CutMix randomly cuts and pasts patches among training images, while the ground truth labels are also mixed proportionally to the area of the patches.
\emph{MaxUp}+CutMix applies CutMix on one image for $m$ times (cutting different randomly sampled patches), 
and select the worst case to do backpropagation.

We test our method on  ResNet-50, ResNet-101~\citep{he2016resnet}, as well as recent energy-efficient  architectures, including ProxylessNet \citep{cai2018proxylessnas} and EfficientNet~\citep{tan2019efficientnet}.
We resize the images to $600 \times 600$ and $845 \times 845$ for EfficientNet-B7 and EfficientNet-B8, respectively~\citep{tan2019efficientnet}, 
for which we process the images with the data processing pipelines proposed by \citet{touvron2019fixing}.
For the other models, 
the input image size is $224 \times 224$.
To save computation resources,
we only fine-tune the pre-trained models with \emph{MaxUp} for a few epochs.  
We set $m = 4$ for \emph{MaxUp} in the ImageNet-2012 experiments unless indicated otherwise. This means that we optimize the worst case in $4$ augmented samples for each image. 

For ResNet-50, ResNet-101 and ProxylessNets, we train the models for 20 epochs with learning rate $10^{-5}$ and batch size $256$ on 4 GPUs for 20 epochs.
For EfficientNet, we fix the parameters in the batch normalization layers and train the other parameters with learning rate $10^{-4}$ and batch size 1000 for 5 epochs.

As shown in Table~\ref{tab:imagenet}, 
for ResNet-50 and ResNet-101, we achieve the best results among all the data augmentation method.
For EfficientNet-B8, we further improve the state-of-the-art result on ImageNet with no extra data.


\paragraph{ResNet-50 on ImageNet}

Table~\ref{tab:resnet50} compares 
a number of state-of-the-art 
 regularization techniques with 
\emph{MaxUp}+CutMix on ImageNet with 
 ResNet-50.\footnote{All the FLOPS and model size reported in this paper is calculated by \url{https://pypi.org/project/ptflops}.} 
We can see that \emph{MaxUp}+CutMix achieves better performance compared to all the strong data augmentation and regularization baselines.
From Table~\ref{tab:resnet50},
we see that CutMix gives the best top1 error ($78.6\%$) among all the augmentation tasks, but our method further improves  it to $78.9\%$.
DropBlock outperforms all the other methods in terms of the top5 error, but by augmenting CutMix with \emph{MaxUp}, we improve the $94.1\%$ top5 error rate obtained by DropbBlock to $94.2\%$.

\paragraph{More Results on Different Architectures}
 Table~\ref{tab:imagenet} shows the result of ImageNet on 
ResNet-101,  ProxylessNet-CPU/GPU/Mobile~\citep{cai2018proxylessnas} and EfficientNet.
We can see that
\emph{MaxUp} consistently improves the results in all these cases.   
On ResNet-101, it improves the $79.83\%$ baseline to $80.26\%$.
On ProxylessNet-CPU and ProxylessNet-GPU, 
\emph{MaxUp} enhances the $75.32\%$ and $75.08\%$ top1 accuracy to $75.65\%$ and $75.42\%$, respectively.
On ProxylessNet-Mobile, we improve the $76.71\%$ top1 accuracy to $77.17\%$.

For EfficientNet-B7, CutMix enhances the original top1 accuracy $85.0\%$~\citep[by][]{tan2019efficientnet} to $85.22\%$.
\emph{MaxUp} further improves the top1 accuracy to $88.45\%$.
On Fix-EfficientNet-B8, \emph{MaxUp} obtains the state-of-the-art $85.80\%$ top1 accuracy.
The previous state-of-the-art top1 accuracy, $85.50\%$, is achieved by EfficientNet-L2.

\subsection{CIFAR-10 and CIFAR-100}
We test \emph{MaxUp} equipped with Cutout \citep{devries2017cutout} on CIFAR-10 and CIFAR-100, 
and denote it by  \emph{MaxUp}+Cutout.
We conduct our method on several  neural architectures, including  ResNet-110~\cite{he2016resnet}, PreAct-ResNet-110~\citep{he2016identity} and WideResNet-28-10~\citep{BMVC2016_87}.
We set $m=10$ for WideResNet and $m=4$ for the other models.
We use the public code\footnote{The code is downloaded from \url{https://github.com/junyuseu/pytorch-cifar-models}}
and keep their hyper-parameters.

\paragraph{Implementation  Details}
For CIFAR-10 and CIFAR-100,
we use the standard data processing pipeline (mirror+ crop) and train the model with 200 epochs. 
All the results reported in this section are averaged over five runs.

We train the models for 200 epochs on the training set with 256 examples per mini-batch, 
and evaluate the trained models on the test set. 
The learning rate starts at 0.1 and is divided by 10 after 100 and 150 epochs for ResNet-110 and PreAct-ResNet-110. 
For WideResNet-28-10, we follow the settings in the original paper~\citep{BMVC2016_87}, where the learning rate is divided by 10 after 60, 120 and 180 epochs. 
Weight decay is set to $2.5^{-4}$ for all the models, and we do not use dropout.

\begin{table}[t]
    \centering
    \begin{tabular}{c|cc}
        \hline
        Model & + Cutout & + \emph{MaxUp}+Cutout \\
        \hline
        ResNet-110 & 94.84 $\pm$ 0.11 & \bf{95.41 $\pm$ 0.08}\\
        PreAct-ResNet-110 & 95.02 $\pm$ 0.15 & \bf{95.52 $\pm$ 0.06}\\ 
        WideResNet-28-10 & 96.92 $\pm$ 0.16 & \bf{97.18 $\pm$ 0.06}\\
        \hline
    \end{tabular}
    \caption{Test accuracy on CIFAR10 for different  architectures.}
    \label{tab:cifar10}
\end{table}

\begin{table}[t]
    \centering
    \begin{tabular}{c|cc}
        \hline
        Model & + Cutout & + \emph{MaxUp}+Cutout \\
        \hline
        ResNet-110 & 73.64 $\pm$ 0.15 & \bf{75.26 $\pm$ 0.21}\\
        PreAct-ResNet-110 & 74.37 $\pm$ 0.13 & \bf{75.63 $\pm$ 0.26}\\ 
        WideResNet-28-10 & 81.59 $\pm$ 0.27 & \bf{82.48 $\pm$ 0.23}\\
        \hline
    \end{tabular}
    \caption{Test accuracy on CIFAR100 for different  architectures.}
    \label{tab:cifar100}
\end{table}

\paragraph{Results}
The results on CIFAR-10 and CIFAR-100 are summarized in Table~\ref{tab:cifar10} and Table~\ref{tab:cifar100}.
We can see that the models trained using \emph{MaxUp}+Cutout significantly outperform the standard Cutout for all the cases. 

On CIAFR-10, 
\emph{MaxUp} improves the standard Cutout baseline from $94.84\%\pm0.11\%$ to $95.41\%\pm0.08\%$ on ResNet-110.
It also improves the accuracy from $95.02\%\pm0.15\%$ to $95.52\%\pm0.06\%$ on  PreAct-ResNet-110.

On CIFAR-100, \emph{MaxUp} obtains improvements by a large margin.
On ResNet-110 and PreAct-ResNet-110, 
\emph{MaxUp} improves the performance of Cutout from $73.64\%\pm0.15\%$ and $74.37\%\pm0.13\%$ to $75.26\%\pm0.21\%$ and $75.63\%\pm0.26\%$, respectively. 
\emph{MaxUp}+Cutout also improves the standard Cutout from $81.59\% \pm 0.27\%$ to $82.48\% \pm 0.23\%$ on  WideResNet-28-10  on CIFAR-100. 

\paragraph{Ablation Study}
We test \emph{MaxUp} with different sample size $m$ and investigate its impact on the performance on 
ResNet-100 (a relatively small model) and WideResNet-28-10 (a larger model). 

Table~\ref{tab:abaltion} shows the result when we vary the sample size in $m \in \{1, 4, 10, 20\}$. 
Note that \emph{MaxUp} reduces to the na\"ive data augmentation method when $m = 1$.  
As shown in Table~\ref{tab:abaltion},  
\emph{MaxUp} with all $m>1$ can improve the result of standard augmentation ($m=1$). 
Setting $m=4$ or $m=10$ achieves best performance on ResNet-110 
, and $m=10$ obtains best performance on WideResNet-28-10. 
We can see that the results are not sensitive once $m$ is in a proper range (e.g., $m\in [4:10]$), and it is easy to outperform the standard data augmentation $(m=1)$ without much tuning of $m$. 
Furthermore,
we suggest to use a large $m$ for large models, and a small $m$  for  relatively small models. 

\begin{table}[t]
    \centering
    \begin{tabular}{c|c|c}
        \hline
        $m$ & ResNet-110 & WideResNet-28-10 \\
        \hline
        1 & 73.64 $\pm$ 0.15 & 81.59 $\pm$ 0.27\\
        4 & 75.26 $\pm$ 0.21 & 81.82 $\pm$ 0.22\\ 
        10 & 75.19 $\pm$ 0.13 & \bf{82.48 $\pm$ 0.23}\\
        20 & 74.37 $\pm$ 0.18 & 82.43 $\pm$ 0.24\\
        \hline
    \end{tabular}
    \caption{Test accuracy on CIFAR100 with ResNet-110 and 
    WideResNet-28-10, when the sample size $m$ varies.} 
    \label{tab:abaltion}
    \vspace{-5pt}
\end{table}

\begin{table*}[tbp]
    \begin{center}
    \setlength{\tabcolsep}{1.2em}
        \begin{tabular}{l|c||cc}
            \hline
            Method & \tt Params & Valid & Test \\
            \hline
            NAS-RNN~\citep{zoph2016neural} & 54M & - & 62.40 \\
            AWD-LSTM~\citep{merity2017awd} & 24M & 58.50 & 56.50 \\
            AWD-LSTM + FRAGE~\citep{gong2018frage} & 24M & 58.10 & 56.10 \\
            AWD-LSTM + MoS~\citep{yang2018breaking} & 22M & 56.54 & 54.44\\
            \hline
            \multicolumn{4}{l}{w/o dynamic evaluation} \\ 
            \hline
            ADV-AWD-LSTM~\citep{wang2019improving} & 24M & 57.15 & 55.01\\
            \bf{ADV-AWD-LSTM + \emph{MaxUp}} & 24M & \bf{56.25} & \bf{54.27}\\
            \hline
            \multicolumn{4}{l}{+~~dynamic evaluation \citep{krause2017dynamic}} \\ 
            \hline
            ADV-AWD-LSTM~\citep{wang2019improving} & 24M & 51.60 & 51.10\\
            \bf{ADV-AWD-LSTM + \emph{MaxUp}}  & 24M & \bf{50.83} & \bf{50.29}\\
            \hline
        \end{tabular}
    \end{center}
\caption{\label{tab:ptb}
Perplexities on the validation and test sets on the Penn Treebank dataset.  Smaller perplexities refer to better language modeling performance. {\tt Params} denotes the number of model parameters.
}
\end{table*}

\begin{table*}[h]
    \begin{center}
    \setlength{\tabcolsep}{1.2em}
        \begin{tabular}{l|c||cc}
            \hline
            Method & \tt Params & Valid & Test \\
            \hline
            AWD-LSTM~\citep{merity2017awd} & 33M & 68.60 & 65.80 \\
            AWD-LSTM + FRAGE~\citep{gong2018frage} & 33M & 66.50 & 63.40 \\
            AWD-LSTM + MoS~\citep{yang2018breaking} & 35M & 63.88 & 61.45\\
            \hline
            \multicolumn{4}{l}{w/o dynamic evaluation} \\ 
            \hline
            ADV-AWD-LSTM~\citep{wang2019improving} & 33M & 63.68 & 61.34 \\
            \bf{ADV-AWD-LSTM + \emph{MaxUp}} & 33M & \bf{62.48} & \bf{60.19}\\
            \hline
            \multicolumn{4}{l}{+~~dynamic evaluation \citep{krause2017dynamic}} \\ 
            \hline
            ADV-AWD-LSTM~\citep{wang2019improving} & 33M & 42.36 & 40.53\\
            \bf{ADV-AWD-LSTM + \emph{MaxUp}}  & 33M & \bf{41.29} & \bf{39.61}\\
            \hline
        \end{tabular}
    \end{center}
\caption{\label{tab:WT2}
Perplexities on the validation and test sets on the WikiText-2 dataset.  Smaller perplexities refer to better language modeling performance. {\tt Params} denotes the number of model parameters. }
\end{table*}

\subsection{Language Modeling}
For language modeling, we test \emph{MaxUp} on two benchmark datasets: Penn Treebank (PTB) and Wikitext-2 (WT2).
We use the code provided by \citet{wang2019improving} as our baseline\footnote{\url{https://github.com/ChengyueGongR/advsoft}}, which stacks a three-layer LSTM
and implements a bag of regularization and optimization tricks
for neural language modeling proposed by \citet{merity2017awd}, such as weight tying, word embedding drop and Averaged SGD. 

For this task, we apply \emph{MaxUp} using  word embedding dropout~\citep{merity2017awd} as the random data augmentation method. 
Word embedding dropout implements dropout on the embedding matrix at the word level,
where the dropout is broadcasted across all the embeddings of all the word vectors. 
For the selected words,
their embedding vectors are set to be zero vectors.
The other word embeddings in the vocabulary are scaled by $\frac{1}{1 - p}$, where $p$ is the probability of embedding dropout.  

As the word embedding layer serves as the first layer in a neural language model,
we apply \emph{MaxUp} in this layer.
We do feed-forward for $m$ times and select the worst case to do backpropagation for each given sentence.
In this section, we set a small $m=2$ since the models are already well-regularized by other regularization techniques.

\paragraph{Implement Details}
The PTB corpus~\citep{marcus1993ptb} is a standard dataset  for benchmarking language models. It consists of 923k training, 73k  validation and 82k test words. 
We use the processed version provided by \citet{mikolov2010recurrent} that is widely used for PTB.

The WT2 dataset is introduced in  ~\citet{merity2017awd} as an alternative to PTB. It contains pre-processed Wikipedia articles, and the training set contains 2 million words.

The  training  procedure can be decoupled into two stages: 1) optimizing the model with SGD and averaged SGD (ASGD);
2) restarting ASGD for fine-tuning twice.
We apply \emph{MaxUp} in both stages, and report the perplexity scores at the end of the second stage.
We also report the perplexity scores with a recently-proposed post-process method, dynamical evaluation ~\citep{krause2017dynamic} after the training process.

\paragraph{Results on PTB and WT2}
The results on PTB and WT2 corpus are illustrated in Table~\ref{tab:ptb} and Table~\ref{tab:WT2}, respectively.
We calculate the perplexity on the validation and test set for each method to evaluate its performance.
We can see 
that \emph{MaxUp}  outperforms the state-of-the-art results achieved by 
Frage~\citep{gong2018frage} and Mixture of SoftMax~\citep{yang2018breaking}. 
We further compare \emph{MaxUp} to the result of ~\citet{wang2019improving} based on AWD-LSTM~\citep{merity2017awd} at two checkpoints, with or without dynamic evaluation~\citep{krause2017dynamic}.
On PTB, we enhance the baseline from $55.01/51.10$ to $54.27/50.29$ at these two checkpoints on the test set.
On WT2, we enhance the baseline from $61.34/40.53$ to $60.19/39.61$ at these two checkpoints on the test set.
Results on validation set are reported in both Table~\ref{tab:ptb} and \ref{tab:WT2} 
to show that the improvement can not achieved by simple hyper-parameter tuning on the test set.

\subsection{Adversarial Certification}
\label{sec:adversarial}
\begin{table*}[htbp]
	\begin{center}
		\begin{tabular}{l|ccccccccccc}
			\hline
            $\ell_2$ RADIUS (CIFAR-10) & $0.25$ & $0.5$ & $0.75$ & $1.0$ & $1.25$ & $1.5$ & $1.75$ & $2.0$ & $2.25$ & $2.5$ & $2.75$ \\
            \hline
            \citet{cohen2019certified} (\%) & 60 & 43 & 34 & 23 & 17 & 14 & 12 & 10 & 8 & 6 & 4\\
            \citet{salman2019provably} (\%) & \bf{74} & \bf{57} & 48 & 38 & 33 & 29 & 25 & 19 & 17 & 14 & 12\\
            \hline
            Ours (\%) & \bf{74} & \bf{57} & \bf{49} & \bf{40} & \bf{35} & \bf{31} & \bf{27} & \bf{22} & \bf{19} & \bf{17} & \bf{15}\\
            \hline
		\end{tabular}
	\end{center}
	\caption{\label{tab:l2-cifar10} \small Certified accuracy on CIFAR-10 of the best classifiers by different methods, evaluated against $\ell_2$ attacks of  
	 different radiuses.}  
	\label{tab:certified}
\end{table*}

Modern image classifiers are known to be sensitive to small, adversarially-chosen perturbations on inputs~\citep{goodfellow2014explaining}.
Therefore, for making high-stakes decisions,  
it is of critical importance to develop methods with \emph{certified robustness}, which 
provide (high probability) provable guarantees 
on the correctness of the prediction subject to arbitrary attacks within certain perturbation ball. 

Recently, \citet{cohen2019certified} proposed 
 to construct certifiably robust classifiers against $\ell_2$ attacks 
 by introducing Gaussian smoothing on the inputs,  
which is shown to outperform all the previous $\ell_2$-robust classifiers in CIFAR-10. 
There has been two major methods for training such smoothed classifiers: 
\citet{cohen2019certified} trains the classifier with  a Gaussian data augmentation technique, while   
\citet{salman2019provably} improves the original Gaussian data augmentation by using PGD (projected gradient descent) adversarial training, 
in which PGD is used to find a local maximal within a given $\ell_2$ perturbation ball.

In our experiment, 
we use \emph{MaxUp} with Gaussian perturbation (referred to as \emph{MaxUp}+Gauss) to train better smoothed classifiers 
than the methods by \citet{cohen2019certified} and \citet{salman2019provably}. 
Like how \emph{MaxUp} improves upon standard data augmentation, 
it is natural to expect that our \emph{MaxUp}+Gauss can learn more robust classifiers than the standard Gaussian data augmentation method in \citet{cohen2019certified}. 



\paragraph{Training Details}
We applied ~\emph{MaxUp} to 
Gaussian augmented data on CIFAR-10 with ResNet-110~\citep{he2016resnet}.  
We follow the training pipelines described in ~\citet{salman2019provably}.
We set a batch size of 256, an initial learning rate of 0.1 which drops by a factor of 10 every 50 epochs, and train the models for 150 epochs.

\paragraph{Evaluation}
After training the smoothed classifiers, 
we evaluation the certified accuracy of different models under different $\ell_2$ perturbation sets.
Given an input image $\boldsymbol{x}$ and a perturbation region $\mathcal{B}$,
the smoothed classifier is called  
certifiably correct if its prediction is correct and has a guaranteed lower bound larger than $0.5$ in $\mathcal{B}$. 
The certified accuracy is the percentage of images 
that are certifiably correct. 
Following \citet{salman2019provably}, 
we calculate the certified accuracy of all the classifiers for various radius and report the best results overall of the classifiers.
We use the codes provided by ~\citet{cohen2019certified} to calculate certified accuracy.\footnote{\url{https://github.com/locuslab/smoothing}}

Following \citet{salman2019provably}, 
we select the best hyperparameters with grid search.
The only two hyperparameters of our \emph{MaxUp}+Gauss are the sample size $m$ and the variance $\sigma^2$ of the Gaussian perturbation, which we search in 
$ m \in \{5, 25, 50, 100, 150\}$ and $\sigma \in \{0.12, 0.25, 0.5, 1.0\}$.
In comparison, 
\citet{salman2019provably} requiers to search a larger number of hyper-parameters, including  
the number of steps of the PGD, 
the number of noise samples,  
the maximum $\ell_2$ perturbation,  
and the variance of Gaussian data augmentation during training and testing. 
Overall, \citet{salman2019provably} requires to train and evaluate  over 150 models for hyperparmeter tuning, 
while \emph{MaxUp}+Gauss  requires only 20 models.

\paragraph{Results}
%
We show the certified accuraries on CIFAR-10 in Table~\ref{tab:certified} under  $\ell_2$ attacks for each $\ell_2$ radius. 
We find that \emph{MaxUp} outperforms  \citet{cohen2019certified} for all the $\ell_2$ radiuses by a large margin.
For example, \emph{MaxUp} can improve the certified accuracy at radius $0.25$ from 60\% to 74\% and improve the 4\% accuracy on  radius $2.75$ to 15\%.
\emph{MaxUp}  also outperforms the PGD-based adversarial training of \citet{salman2019provably} for all the radiuses, boosting the accuracy from 14\% to 17\% at radius $2.5$, and from 12\% to 15\% at radius $2.75$.


In summary, 
\emph{MaxUp} clearly outperforms both \citet{cohen2019certified}  
and 
\citet{salman2019provably}. 
\emph{MaxUp} is also much faster and requires less hyperparameter tuning than \citet{salman2019provably}.   
Although the PGD-based method of \citet{salman2019provably} was designed to outperform the original method by  \citet{cohen2019certified}, 
\emph{MaxUp}+Gauss further  improves upon \citet{salman2019provably}, likely because \emph{MaxUp} with Gaussian perturbation is more compatible with the Gaussian smoothing based certification of \citet{cohen2019certified} than PGD adversarial optimization. 


\section{Conclusion}
In this paper, 
we propose \emph{MaxUp}, a simple and efficient training algorithms for improving generalization,   
especially for deep neural networks.
\emph{MaxUp} can be viewed as a introducing a gradient-norm smoothness regularization for Gaussian perturbation, but does not require to evaluate the gradient norm explicitly,  and 
can be easily combined with any existing data augmentation methods. 
We empirically show that 
\emph{MaxUp} can improve the performance of data augmentation methods in image classification, language modeling, and certified defense. 
Especially, we achieve SOTA performance on ImageNet.

For future works, we will apply \emph{MaxUp} to more applications and models, such as  BERT~\citep{devlin2018bert}.
Furthermore, 
we will generalize \emph{MaxUp} to apply mild adversarial optimization on  feature and label spaces  for other challenging tasks in machine learning, including transfer learning, semi-supervised learning.


\bibliographystyle{icml2020}
\bibliography{references}
\end{document}